\RequirePackage[l2tabu,orthodox]{nag}
\documentclass
[letterpaper,12pt,]
{article}

\usepackage{etex}
\usepackage{verbatim}
\usepackage{xspace,enumerate}
\usepackage[dvipsnames]{xcolor}
\usepackage[T1]{fontenc}
\usepackage[full]{textcomp}
\usepackage[american]{babel}
\usepackage{mathtools}
\usepackage{amsthm}
\usepackage[
letterpaper,
top=1in,
bottom=1in,
left=1in,
right=1in]{geometry}
\usepackage{newpxtext} %
\usepackage{textcomp} %
\usepackage[varg,bigdelims]{newpxmath}
\usepackage[scr=rsfso]{mathalfa}%
\usepackage{bm} %
\let\mathbb\varmathbb
\usepackage{microtype}
\usepackage[pagebackref,colorlinks=true,urlcolor=blue,linkcolor=blue,citecolor=OliveGreen]{hyperref}
\usepackage[capitalise,nameinlink]{cleveref}
\crefname{lemma}{Lemma}{Lemmas}
\crefname{fact}{Fact}{Facts}
\crefname{theorem}{Theorem}{Theorems}
\crefname{corollary}{Corollary}{Corollaries}
\crefname{claim}{Claim}{Claims}
\crefname{example}{Example}{Examples}
\crefname{algorithm}{Algorithm}{Algorithms}
\crefname{problem}{Problem}{Problems}
\crefname{definition}{Definition}{Definitions}
\crefname{exercise}{Exercise}{Exercises}
\crefname{model}{Model}{Models}
\usepackage{amsthm}

\newtheorem{theorem}{Theorem}[section]
\newtheorem*{theorem*}{Theorem}
\newtheorem{lemma}[theorem]{Lemma}
\newtheorem*{lemma*}{Lemma}
\newtheorem{fact}[theorem]{Fact}
\newtheorem*{fact*}{Fact}
\newtheorem{proposition}[theorem]{Proposition}
\newtheorem*{proposition*}{Proposition}

\newtheorem*{corollary*}{Corollary}

\newtheorem*{hypothesis*}{Hypothesis}
\newtheorem{conjecture}[theorem]{Conjecture}
\newtheorem*{conjecture*}{Conjecture}
\theoremstyle{definition}
\newtheorem{definition}[theorem]{Definition}
\newtheorem*{definition*}{Definition}

\newtheorem*{construction*}{Construction}

\newtheorem*{example*}{Example}
\newtheorem{question}[theorem]{Question}
\newtheorem*{question*}{Question}

\newtheorem*{algorithm*}{Algorithm}
\newtheorem{assumption}[theorem]{Assumption}
\newtheorem*{assumption*}{Assumption}

\newtheorem*{problem*}{Problem}

\newtheorem*{openquestion*}{Open Question}
\theoremstyle{remark}

\newtheorem*{claim*}{Claim}

\newtheorem*{remark*}{Remark}

\newtheorem*{observation*}{Observation}
\theoremstyle{model}
\newtheorem{model}[theorem]{Model}
\newtheorem*{model*}{Model}
\usepackage{paralist}
\let\originalleft\left
\let\originalright\right
\renewcommand{\left}{\mathopen{}\mathclose\bgroup\originalleft}
\renewcommand{\right}{\aftergroup\egroup\originalright}
\usepackage{turnstile}
\usepackage{mdframed}
\usepackage{tikz}
\usetikzlibrary{shapes,arrows}
\usetikzlibrary{positioning}
\usetikzlibrary{decorations.pathreplacing}
\usepackage{caption}
\DeclareCaptionType{Algorithm}
\usepackage{newfloat}
\usepackage{array}
\usepackage{subfig}
\usepackage{xparse}
\usepackage{amsthm} %
\makeatletter
\let\latexparagraph\paragraph
\RenewDocumentCommand{\paragraph}{som}{%
  \IfBooleanTF{#1}
    {\latexparagraph*{#3}}
    {\IfNoValueTF{#2}
       {\latexparagraph{\maybe@addperiod{#3}}}
       {\latexparagraph[#2]{\maybe@addperiod{#3}}}%
  }%
}
\newcommand{\maybe@addperiod}[1]{%
  #1\@addpunct{.}%
}
\makeatother

\usepackage{boxedminipage}
\newcommand{\paren}[1]{(#1)}
\newcommand{\Paren}[1]{\left(#1\right)}

\newcommand{\Brac}[1]{\left[#1\right]}

\newcommand{\Bigbrac}[1]{\Big[#1\Big]}

\newcommand{\Abs}[1]{\left\lvert#1\right\rvert}

\newcommand{\Set}[1]{\left\{#1\right\}}

\newcommand{\norm}[1]{\lVert#1\rVert}
\newcommand{\Norm}[1]{\left\lVert#1\right\rVert}

\newcommand{\iprod}[1]{\langle#1\rangle}

\newcommand{\Esymb}{\mathbb{E}}
\newcommand{\Psymb}{\mathbb{P}}

\DeclareMathOperator*{\E}{\Esymb}

\DeclareMathOperator*{\ProbOp}{\Psymb}
\renewcommand{\Pr}{\ProbOp}

\newcommand\bdot\bullet

\DeclareMathOperator{\poly}{poly}

\DeclareMathOperator{\polylog}{polylog}

\newcommand{\iid}{i.i.d.\xspace}

\newcommand{\N}{\mathbb N}
\newcommand{\R}{\mathbb R}

\newcommand{\cP}{\mathcal P}
\newcommand{\cQ}{\mathcal Q}

\newcommand{\cS}{\mathcal S}

\renewcommand{\leq}{\leqslant}

\renewcommand{\geq}{\geqslant}
\renewcommand{\ge}{\geqslant}
\let\epsilon=\varepsilon
\numberwithin{equation}{section}
\newcommand\MYcurrentlabel{xxx}
\newcommand{\MYstore}[2]{%
  \global\expandafter \def \csname MYMEMORY #1 \endcsname{#2}%
}
\newcommand{\MYload}[1]{%
  \csname MYMEMORY #1 \endcsname%
}
\newcommand{\MYnewlabel}[1]{%
  \renewcommand\MYcurrentlabel{#1}%
  \MYoldlabel{#1}%
}
\newcommand{\MYdummylabel}[1]{}
\newcommand{\torestate}[1]{%
  \let\MYoldlabel\label%
  \let\label\MYnewlabel%
  #1%
  \MYstore{\MYcurrentlabel}{#1}%
  \let\label\MYoldlabel%
}
\newcommand{\restatetheorem}[1]{%
  \let\MYoldlabel\label
  \let\label\MYdummylabel
  \begin{theorem*}[Restatement of \cref{#1}]
    \MYload{#1}
  \end{theorem*}
  \let\label\MYoldlabel
}
\newcommand{\restatelemma}[1]{%
  \let\MYoldlabel\label
  \let\label\MYdummylabel
  \begin{lemma*}[Restatement of \cref{#1}]
    \MYload{#1}
  \end{lemma*}
  \let\label\MYoldlabel
}
\newcommand{\restateprop}[1]{%
  \let\MYoldlabel\label
  \let\label\MYdummylabel
  \begin{proposition*}[Restatement of \cref{#1}]
    \MYload{#1}
  \end{proposition*}
  \let\label\MYoldlabel
}
\newcommand{\restatefact}[1]{%
  \let\MYoldlabel\label
  \let\label\MYdummylabel
  \begin{fact*}[Restatement of \cref{#1}]
    \MYload{#1}
  \end{fact*}
  \let\label\MYoldlabel
}
\newcommand{\restatedefinition}[1]{%
  \let\MYoldlabel\label
  \let\label\MYdummylabel
  \begin{definition*}[Restatement of \cref{#1}]
    \MYload{#1}
  \end{definition*}
  \let\label\MYoldlabel
}
\newcommand{\restate}[1]{%
  \let\MYoldlabel\label
  \let\label\MYdummylabel
  \MYload{#1}
  \let\label\MYoldlabel
}

\newcommand{\sse}{\subseteq}

\allowdisplaybreaks
\newcommand*{\Id}{\mathrm{Id}}

\newcommand{\secret}{x^*}
\newcommand{\design}{\bm A}

\newcommand{\estimator}{\hat{\bm x}}

\newcommand{\nspca}{\mathrm{NegSPCA}}
\newcommand{\pairnspca}{\mathrm{PairNegSPCA}}
\newcommand{\concat}{\operatorname{concat}}

\title{
  Computational-Statistical Gaps for Improper Learning in Sparse Linear Regression
}

 \author{
   Rares-Darius Buhai\thanks{ETH Z\"urich.}
   \and
   Jingqiu Ding\footnotemark[1]
   \and
   Stefan Tiegel\footnotemark[1]
}

\begin{document}

\pagestyle{empty}

\maketitle
\thispagestyle{empty} %

\begin{abstract}
    We study computational-statistical gaps for improper learning in sparse linear regression.
More specifically, given $n$ samples from a $k$-sparse linear model in dimension $d$, we ask what is the minimum sample complexity to efficiently (in time polynomial in $d$, $k$, and $n$) find a potentially dense estimate for the regression vector that achieves non-trivial prediction error on the $n$ samples.
Information-theoretically this can be achieved using $\Theta(k \log (d/k))$ samples.
Yet, despite its prominence in the literature, there is no polynomial-time algorithm known to achieve the same guarantees using less than $\Theta(d)$ samples without additional restrictions on the model.
Similarly, existing hardness results are either restricted to the proper setting, in which the estimate must be sparse as well, or only apply to specific algorithms.

We give evidence that efficient algorithms for this task require at least (roughly) $\Omega(k^2)$ samples.
In particular, we show that an improper learning algorithm for sparse linear regression can be used to solve sparse PCA problems (with a negative spike) in their Wishart form, in regimes in which efficient algorithms are widely believed to require at least $\Omega(k^2)$ samples.
We complement our reduction with low-degree and statistical query lower bounds for the sparse PCA problems from which we reduce.

Our hardness results apply to the (correlated) random design setting in which the covariates are drawn \iid from a mean-zero Gaussian distribution with unknown covariance.
\end{abstract}

\clearpage

\microtypesetup{protrusion=false}
\tableofcontents{}
\microtypesetup{protrusion=true}

\clearpage

\pagestyle{plain}
\setcounter{page}{1}

\section{Introduction}

We study computational-statistical gaps in sparse linear regression models with (correlated) random designs.
In particular, on receiving $n$ samples $(\bm a_i, \bm y_i)$ for $\bm y_i = \iprod{\bm a_i, x^*} + \bm w_i$ where $\bm w_i$ is Gaussian noise and $x^*$ is an unknown sparse vector, we wish to find an estimate $\hat{\bm x}$ for $x^*$ that produces predictions $\langle \bm a_i, \hat{\bm x}\rangle$ close to $\langle \bm a_i, x^*\rangle$ for the given samples.
We consider the setting in which the covariates $\bm a_i$ are drawn \iid from a (correlated) Gaussian distribution (cf.~\Cref{def:SLR_model_gauss_design}) -- this is referred to as the \emph{random design} setting.

Despite its prominence, the computational complexity of sparse linear regression with respect to general estimators remains poorly understood.
Without further restrictions on the model, no efficient algorithms are known that use $o(d)$ samples.
Similarly, there is a dearth of hardness results:
Known lower bounds apply only against specific algorithms or rule out so-called \emph{proper} learners, in which the output of the estimator needs to be sparse as well.
In particular, to the best of our knowledge, there is no evidence ruling out efficient algorithms producing potentially dense estimates that use only the information-theoretically minimal number of samples.
We call such estimators \emph{improper}.
The distinction between proper and improper learners can be substantial:
There are learning tasks that are $\mathrm{NP}$-hard in the proper setting but polynomial-time solvable using an improper learner~\cite{valiant1984theory,pitt1988computational}.\footnote{The learning tasks are 3-Term DNFs that are known to be efficiently learnable via 3-CNFs yet are $\mathrm{NP}$-hard to learn properly.}
Thus, hardness against proper learners only gives limited evidence for the hardness of a learning task.

Similarly, many of the known lower bounds only apply to the more stringent setting in which the covariates are allowed to be worst-case (worst-case designs)~\cite{zhang2014LBparameter,zhang2015optimal}.
A priori the random design setting may seem much more benign than the worst-case setting.
Computational hardness results for worst-case designs, albeit restricted to proper learners, have been known for nearly a decade~\cite{zhang2014LBparameter}, yet obtaining any computational lower bounds for random Gaussian designs is a major open question explored in \cite{kelner2022power,kelner2022lower}.
In this work, we show hardness results for improper learners that hold even in the restrictive random design setting.

More specifically, we consider the following model:
\begin{model}[Sparse linear regression with Gaussian design]
    \label{def:SLR_model_gauss_design}
    Let $d,k,n \in \N$ and $\sigma \in \R_{\geq 0}$ be known.
    Let $\Sigma$ be an unknown positive semi-definite matrix and $x^* \in \R^d$ be some unknown $k$-sparse vector, i.e., that has at most $k$ non-zero entries.
    We draw $n$ \iid samples $(\bm a_1,\bm y_1), \ldots, (\bm a_n,\bm y_n) \in \R^d \times \R$ from the following linear model:
    Independently draw $\bm a_i \sim N(0,\Sigma)$ and $\bm w_i \sim N(0,\sigma^2)$.
    Output the $(\bm a_i,\bm y_i)$, where $\bm y_i = \iprod{\bm a_i, x^*} + \bm w_i$.

    Let $\bm A \in \R^{n \times d}$ be the matrix that has the $\bm a_i$ as rows and let $\bm y \in \R^n$ be the vector that has entries $\bm y_i$.
    We say an algorithm achieves \emph{prediction error} $\rho$ if it outputs a (potentially dense) vector $\hat{\bm x} \in \R^d$ such that $\tfrac 1 n \norm{\bm Ax^* - \bm A  \hat{\bm x}}^2 \leq \rho$.
\end{model}
We refer to the matrix $\bm A$ as the design matrix.
In what follows, we will focus on algorithms achieving prediction error $\rho = 0.01$.

Information-theoretically, this problem is well-understood: $\Theta\paren{k \log(d/k)}$ samples are both necessary and sufficient to achieve prediction error $0.01$ with, say, probability 0.99~\cite{raskutti2011minimax}.
Unfortunately, the algorithm achieving the upper bound -- known as the best-subset-selection estimator -- relies on an exhaustive search over $k$-sparse vectors and requires time $d^{\Omega(k)}$.
This exponential-in-$k$ running time is prohibitively large as soon as $k$ is only slightly larger than constant.
Instead, we would like algorithms running in time $\poly(d,n,k)$.
We call such algorithms efficient.
All efficient algorithms for this task use $\Omega(d)$ samples, yet there is no hardness result giving evidence that even $\omega(k \log(d/k))$ samples are necessary.
This leads us to the main question of this work:
\begin{question}
    \label{question:main}
    What is the best-possible sample complexity for computationally efficient algorithms in~\Cref{def:SLR_model_gauss_design} achieving prediction error 0.01?
\end{question}
We make progress on this question by presenting evidence that efficient algorithms need $\Omega(k^2)$ samples, even in the improper setting.
This leads to an intriguing state of affairs:
Many problems involving sparsity exhibit a $k$-vs-$k^2$ statistical-computational gap.
That is, information-theoretically the problem can be solved with (roughly) $k$ samples, yet computationally efficient algorithms likely require $k^2$ samples (we refer to~\cite{brennan2020reducibility} and the references therein for a more detailed treatment).
Our results provide evidence that this might also be the case for sparse linear regression.\footnote{However, we remark that for sparse linear regression there is no known algorithm using $O(k^2)$ or even $o(d)$ samples.}
Our hardness results come in the form of reductions from problems that are believed to be computationally intractable.

\paragraph{Known lower bounds and our approach}
Known lower bounds have focused on showing impossibility results for the harder task of (semi-)proper learning in which the algorithm has to output a $k'$-sparse estimate for $k \leq k' \leq k \cdot d^{o(1)}$.
Based on worst-case complexity assumptions,~\cite{zhang2014LBparameter,foster2015variable} show that for worst-case design matrices polynomial-time algorithms require $\omega(k \log d)$ samples.
While this settles~\Cref{question:main} for (semi-)proper learners, this does not say anything about the improper and/or random design case.
To the best of our knowledge, the only known lower bounds for this more general setting apply to (very) restricted families of algorithms \cite{zhang2015optimal,kelner2022power,kelner2022lower}.
While these include popular sparse regression algorithms such as LASSO and variations thereof, this only provides weak evidence of hardness: It remains unclear if an algorithm from a different family can circumvent the results.
In particular, in recent years algorithms based on spectral methods and the sum-of-squares hierarchy have seen an immense success in the design of computationally efficient algorithms for statistical estimation problems.
Such a general class of algorithms is not ruled out.

This calls for hardness evidence against more general classes of algorithms.
Unfortunately, and in contrast to the proper learning setting, there seem to be inherent barriers to basing hardness of improper learning (in average-case problems) on classical worst-case assumptions such as $\mathrm{P} \neq \mathrm{NP}$~\cite{applebaum2008basing}.
By now, the two most prevalent techniques for showing average-case hardness are the following:
First, show unconditional lower bounds in a restricted model of computation (that captures most known algorithms for the problem), such as statistical query (SQ)~\cite{kearns1998efficient,feldman2017statistical} or low-degree algorithms~\cite{hopkins2017efficient,hopkins2017power,HopkinsThesis}.
These algorithms capture most known algorithms for statistical inference problems (with the notable exception of Gaussian elimination and the LLL Algorithm~\cite{zadik2022lattice,diakonikolas2022non}).
For many average-case problems (such as graph recovery problems~\cite{barak2019nearly,hopkins2017efficient}, mean estimation and learning Gaussian Mixture Models~\cite{diakonikolas2017statistical} and spiked matrix models~\cite{hopkins2017power,ding2022subexponentialtime}), these frameworks have been successfully used to trace out a computational phase transition for a broad class of computational problems:
the statistical error rate achieved by existing polynomial-time algorithms can be matched by algorithms in the class, while obtaining better error rate is impossible for these algorithms.
The second approach is to show a reduction from a problem believed to be computationally hard~\cite{brennan2020reducibility,CLWE,gupte2022continuous}.
We remark that the latter can also be interpreted in a positive way as understanding the connections between different computational problems: If, say, we find a better sparse linear regression algorithm, we also find a better algorithm for another problem that we did not know before.

We follow the second approach.
In particular, we show a reduction from a slight variant of a widely studied PCA problem:
\begin{model}[Sparse spiked Wishart model]
    \torestate{
    \label{def:spiked_wishart}
    Let $d, k, n \in \N$ with $d \geq k$ and $\theta \in (0, 1)$.
    We define the $\nspca_{\theta}$ problem as the following distinguishing problem:
    We are given $n$ i.i.d.~samples from either $\cP$ or $\cQ$ defined as follows and want to decide whether they came from $\cP$ or $\cQ$.
    \begin{itemize}
        \item \textbf{Planted  distribution $\mathcal{P}$}: Sample a unit vector $\bm x \in \R^{d+1}$ with $\bm x_1=-\tfrac 1 {\sqrt{k+1}}$ and $\bm x_{\setminus 1}$ sampled uniformly from $k$-sparse $\Set{\pm \tfrac 1 {\sqrt{k+1}},0}^d$ vectors and fix it for the rest of the sampling procedure.
        Produce samples by sampling from $N(0, \Id_{d+1}-\theta \cdot \bm x\bm x^\top)$.
        \item \textbf{Null distribution $\mathcal{Q}$}: $N(0, \Id_{d+1})$.
    \end{itemize}
    }
\end{model}
We remark that, except for the condition that the first coordinate of $\bm x$ under $\cP$ is non-zero, this coincides with the classical problem of sparse PCA in the Wishart model with a negative spike that has a sparse prior.
Several variants of this problem are known to be hard using (roughly) less than $k^2$ samples: in the low-degree model, inside the sum-of-squares framework, and under a reduction from a variant of the planted clique assumption~\cite{hopkins2017power,ComplexityRIP,brennan2020reducibility}.
We make the following hardness assumption about~\Cref{def:spiked_wishart}, which we formally verify in the low-degree and statistical query model in~\Cref{sec:LD_lower_bound_SPCA,sec:SQWishart} (our proof follows closely the proof of~\cite{ComplexityRIP}).
\begin{assumption}\label{conj:hardSSW}
    Let $d,n,k \in \N$ with $d \geq k$ and let $0 < \delta \leq 0.1$ be an arbitrary absolute constant. 
    Suppose that $n = o(\min(d,k^{2-\delta}))$.
    Then, for any $\theta \in (0, 1)$, any algorithm that solves $\nspca_{\theta}$ in dimension $d$ using $n$ samples and has success probability at least $1-o(1)$ requires running time $d^{k^{\Omega(1)}}$.
\end{assumption}

\paragraph{Results}
We show that a polynomial-time algorithm for improper learning in sparse linear regression that uses $n$ samples can be turned into a polynomial-time algorithm for~\Cref{def:spiked_wishart} using the same number of samples.
In particular, our main result is the following:
\begin{theorem}[Main result, see reduction in~\Cref{thm:main_known_no_RE}]
    \label{thm:main}
    Let $d,n,k \in \N$ with $d \geq k$ and let $0 < \delta \leq 0.1$ be an arbitrary absolute constant.
    Suppose that $n = o(\min(d,k^{2-\delta}))$.
    If there is an improper learner for the Sparse Linear Regression Model with Gaussian Design (cf.~\Cref{def:SLR_model_gauss_design}) that uses $n$ samples and runs in time $d^{k^{o(1)}}$ and achieves prediction error better than 0.01 with probability at least $1-o(1)$, then~\Cref{conj:hardSSW} is false.
\end{theorem}
Note that this result implies that efficient algorithms for improperly learning sparse linear regression models likely need (roughly) $\Omega(\min(k^2,d))$ samples, whereas information-theoretically, only $O(k\log(d/k))$ samples are necessary.
Without additional assumptions on the model (cf.~\Cref{sec:other_hardness_results}) no efficient algorithm is known that uses $o(d)$ samples for this task.
We remark that when we only want to rule out algorithms that achieve prediction error 0.01 with probability at least $1-\Omega(\tfrac 1 d)$, we can reduce from a version of~\Cref{def:spiked_wishart} in which $\bm x$ is sampled uniformly from all $(k+1)$-sparse vectors in $\Set{\pm \tfrac 1 {\sqrt{k+1}},0}$ (without the first coordinate being known).\footnote{We remark that we cannot combine this last reduction with the reductions in~\cite{brennan2020reducibility} to obtain a reduction from (secret-leakage) planted clique. The reason is that the parameters in the reduction from the latter to~\Cref{def:spiked_wishart} are not sufficiently strong. Specifically,~\cite{brennan2020reducibility} requires $\theta = o(1)$, while we need $\theta$ to be an absolute constant close to $1$.}

It turns out that LASSO can solve our hard instance with $O(k^2 \log d)$ samples, so our lower bound is tight up to logarithmic factors. 
Therefore a different construction would be needed to obtain even stronger lower bounds.

\subsection{Relation to known algorithms and other hardness results}
\label{sec:other_hardness_results}

We outline here how our result compares to other algorithmic and hardness results that have been obtained in prior works.

As stated before, known algorithms achieve prediction error 0.01 using $o(d)$ samples only with additional assumptions on the model.
In particular, under the assumption that the columns are normalized to have squared norm $n$, the LASSO estimator is known to achieve prediction error 0.01 using $O(\norm{x^*}_1^2 \log d)$ samples~\cite[Theorem 7.20]{wainwright_2019}.
This gives a good sample complexity if the $\ell_1$-norm of the secret is bounded.
Second, if the design matrix is known to satisfy the so-called \emph{restricted eigenvalue condition} (with the same normalization of the columns), LASSO is known to succeed with few samples.
In particular, let $A$ be the design matrix.
Assuming that (roughly) for all $k$-sparse unit vectors $u$ it holds that $\tfrac 1 n \norm{A u}^2 \geq \gamma$, LASSO achieves prediction error 0.01 using $O(\tfrac{k \log d} \gamma)$ samples.
$\gamma$ is referred to as the RE constant.

Further,~\cite{kelner2022power} showed that in the Gaussian design setting that we consider in this paper, the LASSO estimator combined with a preconditioning step achieves nearly optimal sample complexity whenever the dependency structure of the covariates satisfies a specific regularity condition.

For lower bounds, the results closest to us are the following:
\cite{zhang2014LBparameter} showed that, assuming $\mathrm{NP} \not\sse \mathrm{P}_{/\poly}$, the dependence of the RE constant achieved by the LASSO is tight for  all polynomial-time proper learners.
In particular, they construct instance in which the RE constant can be exponentially small and obtain a corresponding lower bound on the sample complexity of proper learners achieving prediction error $0.01$.
Under other complexity theoretic assumptions~\cite{foster2015variable} showed that there is no polynomial-time estimator for this task, even if the output vector is allowed to be $O(k \cdot 2^{\log^{1-\delta}(d)})$-sparse for any constant $\delta > 0$.
Note that this notion is still significantly more restrictive than ours, as it does note rule out outputting even a $d^{0.001}$-sparse vector.
The design matrices in both the above works are worst-case.

\cite{zhang2015optimal} showed that a restricted class of convex M-estimators (including the LASSO estimator) cannot achieve optimal prediction error (under worst-case designs).
While the algorithms they rule out are improper, their result does not make any claims beyond the specific algorithms they consider.
To the best of our knowledge, the only hardness results for (a restricted class of) improper learners in the random design setting are~\cite{kelner2022power,kelner2022lower}.
They ruled out that the LASSO can achieve optimal sample complexity even when combined with a pre-processing step (called pre-conditioning).

\paragraph{Concurrent work}
Concurrent and independent work of~\cite{gupte2024sparse} shows hardness for proper learning under random designs based on worst-case problems in lattices. In particular, they give evidence that in this setting, improving the prediction error bound of LASSO, when expressing the error as a function of the RE constant of the design matrix, would give stronger lattice algorithms. As the information-theoretic prediction error has no dependence on the RE constant and the RE constant can be exponentially small, this gives evidence of a possibly large computational-statistical separation.

Most closely to our result, concurrent and independent work of~\cite{kelner2024lasso} has also observed the connection between negative sparse PCA and sparse linear regression that we make in this paper.
Using this, they deduce a similar $k$-to-$k^2$ gap for sparse linear regression.
Their lower bound holds for the \emph{generalization error} in the (correlated) random design setting, and their reduction is similar to the unknown-variance reduction discussed in the first part of~\Cref{sec:red_SPCA_to_SLR} in our work.
Our result in~\Cref{sec:main} establishes hardness in this setting for the \emph{training error} (i.e., the error on the observed samples), which is often an easier quantity to minimize than the generalization error.
We remark that~\cite{kelner2024lasso} also contains algorithmic upper bounds under additional assumptions on the model.

\section{Technical overview}
\label{sec:red_SPCA_to_SLR}

\paragraph{Notation}
We use asymptotic notations $O(\cdot),o(\cdot),\Omega(\cdot),\omega(\cdot)$ for $n\to \infty$. 
We use $\tilde{O}(\cdot)$ to hide $\frac{1}{\polylog(n)}$ factors.
We use $\|\cdot\|$ for the Euclidean norm of vectors.  
We use boldface for random variables. For a vector $x \in \R^d$, we use $x_{\setminus 1} \in \R^{d-1}$ for the vector obtained by removing the first coordinate. For a matrix $A \in \R^{n \times d}$, we use $A_j$ for the $j$-th column and $A_{j:k}$ for the submatrix containing columns $j$ through $k$. For two vectors $x \in \R^{d_1}$ and $y \in \R^{d_2}$, we denote by $\concat(x, y)$ the vector $(x, y) \in \R^{d_1+d_2}$.

\vspace{1em}

In this section, we present our reduction from PCA with a negative spike (cf.~\cref{conj:hardSSW}) to improperly learning sparse linear regression models.
Our reduction is similar to the one of \cite{breslerSLR}.
However, a crucial difference is that we start from sparse PCA with a \emph{negative} spike.
Indeed, since the design matrices in the resulting SLR instance in \cite{breslerSLR} satisfy the restricted eigenvalue property with appropriate parameters, they can be solved with error 0.01 using only $O(k \log d) \ll k^2$ samples.
Further, the reduction of \cite{breslerSLR} only works assuming access to a \emph{proper} sparse linear regression algorithm, i.e., one that outputs a sparse solution.
By giving a slightly different reduction, we show that it is enough to assume access to an improper learner.
As a warm-up, we will show that improperly learning sparse linear regression models is hard when the variance of the noise is unknown (but known to be in [0,1]).
Second, we show that even when the variance is known and equal to 1  (i.e., \cref{def:SLR_model_gauss_design}) the problem remains hard.
For the remainder of the paper, unless explicitly specified, when referring to "solving sparse linear regression" or similar expressions, we always mean improper learners.

\paragraph{Hardness of certifying RIP and a first lower bound}
We first recall that achieving prediction error $0.01$ is possible in polynomial time using $O(k \log d)$ samples under additional assumptions on the design matrix.
In particular, if the design matrix satisfies the restricted eigenvalue condition (with a constant), the LASSO algorithm achieves this guarantee.
In order to construct a lower bound instance, it is thus necessary that this property does not hold (or only with weaker parameters than necessary for known algorithms).
We focus here on the related \emph{restricted isometry property} (RIP).
In particular, a matrix $X \in \R^{n\times d}$ is said to satisfy $(k,\delta)$-RIP if for all $k$-sparse unit vectors $v$ it holds that $\norm{Xv}^2 \in [1-\delta,1+\delta]$.

Suppose that we could \emph{certify} that a matrix satisfies RIP.
This would give us a way to convince ourselves that our learning algorithm succeeded in some cases:
Check if the design matrix satisfies RIP, and if yes run LASSO.
Else, we do not guarantee anything.
It is thus natural to wonder whether there is a formal connection between RIP certification and sparse linear regression.
Unfortunately, it is believed that certifying RIP is computationally difficult~\cite{ComplexityRIP}.\footnote{Interestingly, the results in~\cite{ComplexityRIP} are tight, in the sense that there is a certification algorithm matching their lower bound~\cite{koiran2014hidden}.}
On the other hand, our work shows that we can exploit this connection to show \emph{hardness results} in the following way:
Lower bounds for RIP certification are proved by showing that an associated distinguishing problem between two distributions is hard.
One hypothesis corresponds to matrices with good RIP, the other does not.
If we can certify RIP, we can distinguish the two.
Our main observation is that sparse linear regression solvers can be used to solve this distinguishing problem. 

In particular, consider the following (degenerate) instance of the negative sparse PCA problem.
This instance will already give us a lower bound for sparse regression if we ask that the algorithm works in the case when the variance of the noise is unknown between $0$ and $1$.
Given samples $\bm z_1, \ldots, \bm z_n$ from either $\cQ = N(0,I_d)$ or $\cP = N(0,I_d - \bm x \bm x^\top)$, where $\bm x$ is a uniformly random $(k+1)$-sparse unit vector with $\bm x_1 = -\tfrac 1 {\sqrt{k+1}}$, decide which is the case.
Note that this instance being hard implies that RIP certification is hard: Consider the matrix $\bm Z \in \R^{n\times d}$ with rows $\bm z_i$.
Under $\cQ$ this has good RIP and under $\cP$ it does not.
Our key observation is that the absence of RIP implies a dependence between the columns of $\bm Z$ that is not present under $\cQ$.
Luckily for us, in our setting this dependence is linear and only concerns $k$ columns.
We can thus hope to detect it with our sparse regression oracle.

Indeed, consider the matrix $\bm Z$ under $\cP$:
It holds that $\bm Z_1 \bm x_1 + \bm Z_{\setminus 1} \bm x_{\setminus 1} = \bm Z \bm x = 0$ and hence $\bm Z_1 = \bm Z_{\setminus 1} \paren{- \tfrac 1 {\bm x_1} \bm x_{\setminus 1}}$.
Under $\cQ$ on the other hand the columns are independent and $\bm Z_1 = \bm Z_{\setminus 1} \cdot 0 + \bm w$, where $\bm w \sim N(0,\Id_n)$.
In either case, let $\bm A = \bm Z_{\setminus 1}$ and $\bm y = \bm Z_1$.
We have shown that in both cases $(\bm A, \bm y)$ forms a valid input for our regression algorithm since $\bm y = \bm A \bm \secret + \bm w$, where $\bm x^*$ is $0$ or $- \tfrac 1 {\bm x_1} \bm x_{\setminus 1}$ under $\cQ$ and $\cP$ respectively and $\bm w$ is independent of $\bm A$ and either $N(0,\Id_n)$ or $0$.\footnote{Note that the $0$ vector is also $k$-sparse.}
Our regression oracle allows us to estimate $\bm A \bm x^*$.
In particular, under $\cQ$, $\bm A \bm \secret = \bm Z_{\setminus 1} \bm x^* = 0$ and hence if $\bm \estimator$ is our estimator then $\bm A \bm \estimator$ should have small norm.
On the other hand, under $\cP$, $\bm A \bm \secret = \bm Z_{\setminus 1} \bm x^* = \bm Z_1 \sim N(0,(1-\tfrac 1 {k+1})\cdot \Id_n)$ and we expect our estimate to have large norm.
Indeed, suppose our regression oracle outputs $\estimator$ achieving prediction error 0.01 with probability at least $1-\delta$.
Then, under $\cQ$, $\tfrac 1 {\sqrt{n}} \norm{\bm A \estimator } = \tfrac 1 {\sqrt{n}} \norm{\bm A  \estimator - \bm A \bm \secret} \leq 0.1$ with probability at least $1-\delta$.
However, under $\cP$ it holds that
\[
    \tfrac 1 {\sqrt{n}} \Norm{\bm A \estimator } \geq \Abs{\tfrac 1 {\sqrt{n}} \Norm{\bm A \estimator  - \bm A \bm \secret} - \tfrac 1 {\sqrt{n}} \Norm{\bm A \bm \secret}} \,.
\]
By standard concentration bounds it follows that $\tfrac 1 {\sqrt{n}} \Norm{\bm A \bm \secret} = \tfrac 1 {\sqrt{n}} \Norm{\bm Z_1}$ is at least, say, $0.6$ with high probability and thus $\tfrac 1 {\sqrt{n}} \norm{\bm A \estimator } \geq 0.5$ with probability at least $1-\delta-o(1)$.
It follows that thresholding $\tfrac 1 {\sqrt{n}} \norm{\bm A \estimator }$ faithfully distinguishes $\cQ$ and $\cP$ with probability at least $1-\delta-o(1)$.

Note that if we assume that our sparse linear regression algorithm has success probability at least $1-O(\tfrac 1 d)$, we do not need to assume that the first coordinate of the sparse PCA prior in~\cref{def:spiked_wishart} is known to be $-\tfrac1{\sqrt{k+1}}$.
Indeed, instead of the above reduction, we can set $\bm y = \bm Z_i$ and $\bm A = \bm Z_{\setminus i}$ for all $i \in [d]$ and run our regression solver on $(\bm A, \bm y)$.
Under $\cQ$, with probability at least, say, 0.99, it holds that in all of these runs we have that $\tfrac 1 {\sqrt{n}} \norm{\bm A \estimator } \leq 0.1$.
Similarly, under $\cP$, with probability at least 0.99, there exists at least one $i$ such that $\tfrac 1 {\sqrt{n}} \norm{\bm A \estimator } \geq 0.5$.
Thus, we can distinguish $\cQ$ and $\cP$ with at least constant probability, based on whether $\tfrac 1 {\sqrt{n}} \norm{\bm A \estimator }$ is large in at least one iteration. 

Note that in both cases the $\ell_1$-norm of the secret is at most $O(k)$ and thus LASSO would achieve prediction error $0.01$ using $O(k^2 \log d)$ samples, under the slow-rate analysis.

\paragraph{Extension to non-degenerate negative sparse PCA}
So far our reduction used a degenerate planted hypothesis $\cP = N(0,I_d - \bm x \bm x^\top)$.
In what follows, we argue that the same reduction can also produce valid instances from non-degenerate hypotheses.
The sparse linear regression instances produced will now have unknown variance of the noise between $0.5$ and $1$.

In particular, let $\theta = \tfrac{k+1}{k+2} \approx 1 - \tfrac 1 k$.
We start with the negative sparse PCA problem in which $\cQ = N(0,\Id_d)$ and $\cP = N(0,\Id_d - \theta \cdot \bm x\bm x^\top)$, where $\bm x$ is again a uniformly random $(k+1)$-sparse unit vector with $\bm x_1 = -\tfrac 1 {\sqrt{k+1}}$.
The null case does not change (and the variance of the noise in this case is 1), so it suffices to analyze $\cP$.
Then, for the planted case, we use the following fact proved in \cite{breslerSLR}:\footnote{Formally, \cite{breslerSLR} only shows this fact for sparse PCA instances with \emph{positive} spikes. The proof they give also works for the setting with negative spikes.}
\begin{fact}[Appendix B.2 in \cite{breslerSLR}]
    \label{fact:SLR_model_planted}
    Let $\bm y, \design$ be as above (under $\cP$) and $\gamma = \tfrac {\theta}{1-\theta \cdot \tfrac {k}{k+1}} = \tfrac {k+1} 2$. Then $\bm y = \design {\bm\secret} + \bm w$, where $\design {\bm\secret}$ and $\bm w$ are independent with $\bm w \sim N(0,\sigma^2 \Id_n)$, and 
    \begin{align*}
        {\bm\secret} = \frac \gamma {\sqrt{k+1}} \cdot \bm x_{\setminus 1} = \frac {\sqrt{k+1}} 2 \cdot \bm x_{\setminus 1}\,, && \sigma^2 = 1- \frac \gamma {{k+1}}  = \frac 1 2\,.
    \end{align*}
\end{fact}

This establishes that the variance of the noise is always either 0.5 or 1.
It follows using the same techniques that under $\cP$ it still holds that $\tfrac 1 {\sqrt{n}} \norm{\bm A \estimator } \geq 0.5$ with probability at least $1-\delta-o(1)$.

\paragraph{Known variance of the noise} Our goal is to reduce $\nspca_\theta$ to sparse linear regression instances that have \textit{known} variance of the noise. 

In the discussions so far our strategy has been to set $\bm{y} = \bm{Z}_1$.
Then under $\cQ$ we have $\bm{y} \sim N(0, \Id_n)$, but under $\cP$ we have $\bm{y} = \design{\bm\secret} + \bm{w}$ where $\bm{w}$ is Gaussian with variance \textit{less than $1$}.
Indeed, if $\bm{w}$ had variance exactly $1$, we could distinguish $\cQ$ and $\cP$ by thresholding the norm of $\bm{y}$.
Therefore, in this strategy the variance of the noise cannot be the same under $\cQ$ and $\cP$.
More generally, in $\nspca_\theta$ the null case and the planted case are asymmetric, so it seems difficult to obtain some $\bm{y}$ for which the variance of the noise is the same.

Instead, we introduce a new, symmetric, distinguishing problem $\pairnspca_\theta$ which requires distinguishing between a sample from $\cP \times \cQ$ and a sample from $\cQ \times \cP$, where $\cP$ and $\cQ$ are the planted and null distributions in $\nspca_\theta$:

\begin{definition}[Paired spiked Wishart model]
    \torestate{
    \label{def:paired_spiked_wishart}
    Let $d, k, n \in \N$ with $d \geq k$ and $\theta \in (0, 1)$.
    Let $\cP$ and $\cQ$ be the planted and null distributions in a $\nspca_{\theta}$ problem, respectively. We define the $\pairnspca_{\theta}$ problem as the following distinguishing problem:
    We are given $n$ \iid samples from either $\cP \times \cQ$ or $\cQ \times \cP$ defined as follows and we want to decide whether they came from $\cP \times \cQ$ or $\cQ \times \cP$.
    \begin{itemize}
        \item $\cP \times \cQ$: Sample independently $z \sim \cP$ and $z' \sim \cQ$ and return $(z, z')$.
        \item $\cQ \times \cP$: Sample independently $z \sim \cQ$ and $z' \sim \cP$ and return $(z, z')$. 
    \end{itemize}
    }
\end{definition}

Then we reduce $\pairnspca_\theta$ to sparse linear regression by setting $\bm{y}$ to be the sum of a column as in $\cP$ and a column as in $\cQ$ -- without knowing which is which.
This ensures that the variance of the noise is identical under $\cP \times \cQ$ and $\cQ \times \cP$.
In addition, we reduce $\nspca_\theta$ to $\pairnspca_\theta$, so by composition we obtain the desired reduction from $\nspca_\theta$ to sparse linear regression with known variance of the noise.

\section{Sparse linear regression reduction}
\label{sec:main}
\Cref{thm:main_known_no_RE} gives our main result.

\begin{theorem}
    \label{thm:main_known_no_RE}
    Let $d, k, n \in \N$ with $d \geq k$ and $\theta \in (0, 1)$.
    Assume there exists an efficient algorithm that, given $n$ samples from the \hyperref[def:SLR_model_gauss_design]{sparse linear regression model} with variance of the noise $\sigma^2 = 1$, achieves prediction error $\frac{1}{n}\Norm{\design\estimator -  \design \bm\secret}^2$ at most $0.01$ with probability $1-\delta$.
    Then there exists an efficient algorithm that solves $\nspca_{\theta}$ for $\theta = \frac{k+1}{k+2}$ with probability $1-\sqrt{2\delta}-\exp(-\Omega(n))$.
\end{theorem}

We first state a reduction from $\nspca_\theta$ to $\pairnspca_\theta$.
Informally, this says that if it is hard to decide whether a sample comes from $\cP$ or $\cQ$, then it is also hard, given a sample from $\cP$ and one from $\cQ$, to decide which is which.
The result follows from a more general reduction from distinguishing distributions to distinguishing paired distributions that we give in~\Cref{lemma:disinguishing-order}.

\begin{lemma}
\label{lem:red-nspca-to-pairnspca}
Let $d, k, n \in \N$ with $d \geq k$ and $\theta \in (0, 1)$.
Assume there exists an efficient algorithm that solves $\pairnspca_{\theta}$ with probability at least $1-\delta$.
Then there exists an efficient algorithm that solves $\nspca_{\theta}$ with probability at least $1-\sqrt{2\delta}$.
\end{lemma}

The remaining step is to reduce $\pairnspca_\theta$ to sparse linear regression.

\paragraph{The reduction}
Let $\theta = \tfrac {k+1}{k+2}$.
Given $n$ samples $(\bm z_1, \bm z_1'), \ldots, (\bm z_n, \bm z_n')$ from $\pairnspca_\theta$, let $\bm Z \in \R^{n \times 2d}$ be the matrix with rows $\concat(\bm z_1, \bm z_1'), \ldots, \concat(\bm z_n, \bm z_n')$.
We do the following:
\begin{enumerate}
    \item Set $\bm y = {\bm Z}_1 + {\bm Z}_{d+1}$ and $\design = {\bm Z}_{\setminus \{1, d+1\}}$. That is, $\design$ is the $n \times (2d-2)$ matrix that is obtained by removing columns $\bm{Z}_1$ and $\bm{Z}_{d+1}$ from $\bm{Z}$.
    \item Invoke the sparse linear regression solver on $(\design, \bm y)$ to obtain an estimator $\estimator$.
    \item If $\tfrac{1}{\sqrt{n}} \Norm{\design \estimator - {\bm Z}_1} \leq \tfrac{1}{\sqrt{n}} \Norm{\design \estimator - {\bm Z}_{d+1}}$ output $\cP \times \cQ$. Else output $\cQ \times \cP$.
\end{enumerate}

Consider the case $\cP \times \cQ$. As in~\Cref{fact:SLR_model_planted}, let $\gamma = \tfrac {\theta}{1-\theta \cdot k /(k+1)} = \frac{k+1}{2}$.
Then, by the same argument as in~\Cref{fact:SLR_model_planted}, we can write ${\bm Z}_1 = \design {\bm\secret} + \bm w'$, where $\design {\bm\secret}$ and $\bm w'$ are independent with $\bm w' \sim N(0, \sigma^2 \Id_n)$, and 
\begin{align*}
    {\bm\secret} = \tfrac {\sqrt{k+1}} 2 \cdot \concat(\bm x_{\setminus 1}, \bm 0^{d-1})\,, && \sigma^2 = \frac{1}{2} \,.
\end{align*}
Also, ${\bm Z}_{d+1} \sim N(0, \Id_n)$ is independent of $\design {\bm\secret}$  and $\bm{w}'$, so we can write $\bm y = \design {\bm\secret} + \bm w$, where $\design {\bm\secret}$ and $\bm w$ are independent with $\bm w \sim N(0, (\sigma^2 + 1) \Id_n)$.
Note that by symmetry the variance of the noise is $\sigma^2+1$ also in the case $\cQ \times \cP$.

We now prove~\Cref{thm:main_known_no_RE}.

\begin{proof}[Proof of~\Cref{thm:main_known_no_RE}]
Consider the case $\cP \times \cQ$. 
We have ${\bm Z}_1 =\design{\bm\secret}+\bm w$, where $\bm w\sim N(0,1.5\Id_n)$.
By the assumption on our sparse linear regression estimator, by scaling up the guarantees such that the known variance of the noise is $1.5$, we have with probability $1-\delta$ that $\frac{1}{\sqrt{n}} \Norm{\design\estimator - \design{\bm\secret}} \leq 0.01\cdot \sqrt{1.5}\leq 0.13$. 
By the triangle inequality, we get
\begin{align*}
\frac{1}{\sqrt{n}} \Norm{\design\estimator - {\bm Z}_1}
&\leq \frac{1}{\sqrt{n}} \Norm{\design\estimator - \design{\bm\secret}} + \frac{1}{\sqrt{n}} \Norm{\bm{w}'}\,.
\end{align*}
Recall that $\bm{w}'$ has covariance matrix $0.5 \cdot \Id_n$. Then we have by~\Cref{fact:conc_gaussian_norm} with probability $1-\exp(-\Omega(n))$ that $\tfrac{1}{\sqrt{n}} \Norm{\bm w'} \leq \sqrt{0.5}+0.001\leq 0.71$. Then
\[\frac{1}{\sqrt{n}} \Norm{\design\estimator - {\bm Z}_1} \leq 0.13+0.71=0.84\,.\]
On the other hand
\begin{align*}
\frac{1}{\sqrt{n}} \Norm{\design\estimator - {\bm Z}_{d+1}}
&\geq \Abs{\frac{1}{\sqrt{n}} \Norm{\design\estimator - \design{\bm\secret}} - \frac{1}{\sqrt{n}}\Norm{\design{\bm\secret} - {\bm{Z}}_{d+1}}}\,.
\end{align*}
The entries of ${\bm Z}_{d+1}$ are i.i.d. Gaussian with mean zero and variance $1$, independent of the multivariate Gaussian $\design{\bm\secret}$. Then we have by~\Cref{fact:conc_gaussian_norm} with probability $1-\exp(-\Omega(n))$ that $\frac{1}{\sqrt{n}}\Norm{\design{\bm\secret} - {\bm{Z}}_{d+1}}$ is at least $0.99$. Then 
\[\frac{1}{\sqrt{n}} \Norm{\design\estimator - {\bm Z}_{d+1}} \geq 0.99 - 0.13 = 0.86\,,\]
and overall we output $\cP \times \cQ$ with probability at least $1-\delta-\exp(-\Omega(n))$.

Analogously, by symmetry, in the case $\cQ \times \cP$ we also output $\cQ \times \cP$ with probability $1-(\delta+\exp(-\Omega(n)))$. Then, by the reduction in~\Cref{lem:red-nspca-to-pairnspca}, we also solve $\nspca_\theta$ with probability at least $1-\sqrt{2 (\delta+\exp(-\Omega(n)))} \geq 1-\sqrt{2\delta}-\exp(-\Omega(n))$.

\end{proof}
\section{Distinguishing paired distributions reduction}

In this section we show that an algorithm that distinguishes a sample from the paired distributions $\cP \times \cQ$ and $\cQ \times \cP$ also distinguishes a sample from the distributions $\cP$ and $\cQ$, assuming that $\cQ$ is known.

\begin{lemma}
\label{lemma:disinguishing-order}
Let $\mathcal{P}$ and $\mathcal{Q}$ be probability distributions over a domain $\mathcal{D}$, and suppose that it is possible to sample independently from $\mathcal{Q}$ in time $T_{\mathrm{sample}}$.
Suppose that there exists a (randomized) algorithm $\mathcal{A}$ that takes as input a pair $(x, x') \in \mathcal{D} \times \mathcal{D}'$ and runs in time $T_{\mathrm{dist}}$ with the following guarantees:
\begin{itemize}
    \item $\Pr_{\bm x \sim \mathcal{P}, \bm x' \sim \mathcal{Q}}(\mathcal{A}(\bm x, \bm x') = 0) \geq 1-\delta$, where $\bm x$ and $\bm x'$ are sampled independently and where the probability is also over any internal randomness of $\mathcal{A}$.
    \item $\Pr_{\bm x \sim \mathcal{Q}, \bm x' \sim \mathcal{P}}(\mathcal{A}(\bm x, \bm x') = 1) \geq 1-\delta$, where $\bm x$ and $\bm x'$ are sampled independently and where the probability is also over any internal randomness of $\mathcal{A}$.
\end{itemize}
Then there also exists a randomized algorithm $\mathcal{B}$ that takes as input $x \in \mathcal{D}$ and runs in time $O((T_{\mathrm{sample}} + T_{\mathrm{dist}})/\sqrt{\delta})$ with the following guarantees:
\begin{itemize}
    \item $\Pr_{\bm x \sim \mathcal{P}}(\mathcal{B}(\bm x) = 0) \geq 1-\sqrt{2\delta}$, where the probability is also over any internal randomness of $\mathcal{B}$.
    \item $\Pr_{\bm x \sim \mathcal{Q}}(\mathcal{B}(\bm x) = 1) \geq 1-\sqrt{2\delta}$, where the probability is also over any internal randomness of $\mathcal{B}$.
\end{itemize}
\end{lemma}
\begin{proof}
Let $\bm x \in \mathcal{D}$ such that either $\bm x \sim \mathcal{P}$ or $\bm x \sim \mathcal{Q}$.
Let $M$ be a positive integer that we will fix later.
The randomized algorithm $\mathcal{B}$ is the following:
\begin{enumerate}
    \item Initialize $\Delta = 0$.
    \item Repeat $M$ times:
    \begin{enumerate}
        \item Sample $\bm x' \sim \mathcal{Q}$.
        \item If $\mathcal{A}(\bm x, \bm x')$ returns $0$ and $\mathcal{A}(\bm x', \bm x)$ returns $1$, increment $\Delta$.
    \end{enumerate}
    \item If $\Delta = M$, return $0$. Else, return $1$.
\end{enumerate}
We now analyze the algorithm.

If $\bm x \sim \mathcal{P}$, then $\Pr_{\bm x \sim \mathcal{P}}(\mathcal{B}(\bm x)=1) =  \Pr(\Delta < M) \leq 2M\delta$.
This is because the probability that $\mathcal{A}$ is wrong in an iteration is at most $\delta$, so the probability that $\mathcal{A}$ is wrong in any of the $2M$ iterations is at most $2M\delta$.

On the other hand, if $\bm x \sim \mathcal{Q}$, we show that $\Pr_{\bm x \sim \mathcal{Q}}(\mathcal{B}(\bm x)=0)=\Pr(\Delta = M) \leq \frac{1}{M+1}$.
From now on, we condition on the internal randomness of $\mathcal{A}$, which makes $\mathcal{A}$ deterministic, and show that the same probability upper bound holds irrespective of this internal randomness.
This implies that the upper bound also holds unconditionally.
To begin, let $\bm x_1, ..., \bm x_M$ be the $M$ random variables corresponding to the $x'$ that are generated in the $M$ iterations.
Then we are interested in the event $E(\bm x, (\bm x_1, ..., \bm x_M))$ that $\Delta$ is incremented in each iteration.
First, note that $E(\bm x, (\bm x_1, ..., \bm x_M))$ has the same probability as $E(\bm x_i, (\bm x, \bm x_1, ..., \bm x_{i-1}, \bm x_{i+1}, ..., \bm x_M))$ for all $i \in [M]$, because $\bm x, \bm x_1, ..., \bm x_M$ are independently and identically distributed according to $\mathcal{Q}$.
Second, note that all the events $E(\bm x, (\bm x_1, ..., \bm x_M))$ and $E(\bm x_i, (\bm x, \bm x_1, ..., \bm x_{i-1}, \bm x_{i+1}, ..., \bm x_M))$ for all $i \in [M]$ are mutually exclusive
(e.g., for $i \neq j$, $E(\bm x_i, (\bm x, \bm x_1, ..., \bm x_{i-1}, \bm x_{i+1}, ..., \bm x_M))$ implies that $\mathcal{A}(\bm x_i, \bm x_j)$ returns $0$, but $E(\bm x_j, (\bm x, \bm x_1, ..., \bm x_{j-1}, \bm x_{j+1}, ..., \bm x_M))$ implies that $\mathcal{A}(\bm x_i, \bm x_j)$ returns $1$, which is a contradiction).
Hence, the probability that $\Delta$ is incremented in every iteration is at most $\frac{1}{M+1}$.
Note that this argument did not use the value of the internal randomness of $\mathcal{A}$, so the same probability upper bound also holds unconditionally.

Taking $M=1/\sqrt{2\delta}$ gives the desired probability bounds.
Finally, the time complexity of algorithm $\mathcal{B}$ is $O((T_{\mathrm{sample}} + T_{\mathrm{dist}})M) = O((T_{\mathrm{sample}} + T_{\mathrm{dist}})/\sqrt{\delta})$.
\end{proof}

\section*{Acknowledgements}
The authors thank David Steurer for helpful discussions. 
This project has received funding from the European Research Council (ERC) under the European Union’s Horizon 2020 research and innovation programme (grant agreement No 815464).

\addcontentsline{toc}{section}{References}
\bibliographystyle{amsalpha}
\bibliography{bib/mathreview,bib/dblp,bib/custom,bib/scholar}

\appendix

\section{Concentration bounds}
\label{sec:concentration_bounds}

The following fact is taken from \cite[Example 2.11]{wainwright_2019}.
\begin{fact}
    \label{fact:conc_gaussian_norm}
    Let $\bm X \sim N(0,\Id_n)$. Then
    \[
        \Pr\Paren{\Abs{\frac{1}{n} \Norm{\bm X}^2 - 1} \geq t} \leq 2 \exp\Paren{- \frac {nt^2}8}\,.
    \]
\end{fact}

\section{Low-degree lower bound for negative-spike sparse Wishart model}
\label{sec:LD_lower_bound_SPCA}

\subsection{Low-degree framework}\label{sec:appendLDLR}
One significant barrier for proving computational lower bounds in high dimension statistics is that we need to show average-case hardness: we want to show that no efficient algorithms can succeed with high probability over random inputs. 
    It is notoriously difficult to show average-case hardness based on worst-case hardness assumptions such as $\mathrm{P} \neq \mathrm{NP}$, except in a few notable examples~\cite{regev2024lattices,CLWE,brennan2020reducibility}. 

The low-degree method, which is developed in~\cite{hopkins2017efficient,HopkinsThesis,hopkins2017power}, provides a way to give evidence for average-case hardness by showing that no efficient algorithms based on low-degree polynomials can succeed with high probability over random inputs.
In the context of hypothesis testing problems, it has been applied to a wide range of fundamental problems, such as community detection~\cite{hopkins2017efficient}, spiked matrix models~\cite{kunisky2019notes}, sparse principal component analysis~\cite{hopkins2017power, ding2022subexponentialtime}, and certifying restricted isometry property~\cite{ComplexityRIP}.

The crux of the techniques is to bound the projection of the likelihood ratio function onto the space of low-degree polynomials:
\begin{definition}[Definition 1.14 in~\cite{kunisky2019notes}]\label{def:LDLR}
    Let $\mathsf{V}^{\leq D}_n \subset L^2(\mathcal{Q}_n)$ denote the linear subspace of polynomials $\mathsf{S}_n \to \R$ of degree at most $D$.
    Let $\mathsf{P}^{\leq D}: L^2(\mathcal{Q}_n) \to \mathsf{V}^{\leq D}_n$ denote the orthogonal projection onto this linear subspace.\footnote{To clarify, the orthogonal projection is with respect to the inner product induced by $\mathcal{Q}_n$ operator on this subspace.}
    Finally, define the \emph{$D$-low-degree likelihood ratio ($D$-LDLR)} as $L_n^{\leq D} \coloneqq \mathsf{P}^{\leq D} L_n$.
\end{definition}

Then in hypothesis testing a low-degree lower bound shows that efficient algorithms based on thresholding low-degree polynomials are inherently obstructed from distinguishing the two distributions. 
Particularly, we have the following proposition connecting the low-degree likelihood ratio and the distribution separation by low-degree polynomials:
\begin{proposition}[\cite{hopkins2017efficient,HopkinsThesis,kunisky2019notes}]
    \label{prop:LDLR}
    Let $\mathcal{P}$ and $\mathcal{Q}$ be two sequences of probability measures, and $L_n=\frac{dP}{dQ}$ be their likelihood ratio. 
    Then there exists $\varepsilon > 0$ and $D = D(n) \ge (\log n)^{1+\varepsilon}$ for which $\|L_n^{\leq D}\|$ remains bounded as $n \to \infty$ if and only if for any degree $\leq D$ polynomials $f(\cdot)$ we have
    \begin{equation*}
        \frac{\E_{\mathcal{P}} f}{\sqrt{\E_{\mathcal{Q}} f^2}}\leq O(1)\,.
    \end{equation*}
\end{proposition}

The low-degree conjecture states that, for "sufficiently nice" distributions, low-degree lower bounds rule out all polynomial-time algorithms that strongly distinguish the distributions:
\begin{conjecture}[Informal,~\cite{HopkinsThesis,kunisky2019notes}]
    \label{conj:low-deg-informal}
    For "sufficiently nice" sequences of probability measures $\mathcal{P}$ and $\mathcal{Q}$, let $L_n=\frac{dP}{dQ}$ be their likelihood ratio.
        If there exists $\varepsilon > 0$ and $D = D(n) \ge (\log n)^{1+\varepsilon}$ for which $\|L_n^{\leq D}\|$ remains bounded as $n \to \infty$, then there is no polynomial-time algorithm that \emph{strongly distinguishes} $\mathcal{P}$ and $\mathcal{Q}$, i.e., distinguishes the two distributions with $1-o(1)$ probability.
\end{conjecture}

The formal version of the conjecture, given originally as Conjecture 2.2.4 in~\cite{HopkinsThesis} and further updated in~\cite{holmgren2020counterexamples}, only rules out polynomial time-algorithms that distinguish between $U_\delta \mathcal{P}$ and $\mathcal{Q}$, for any $\delta>0$, where $U_\delta$ is the Ornstein-Uhlenbeck noise operator: $U_\delta \mathcal{P}$ is sampled by drawing $x \sim \mathcal{P}$ and $y \sim \mathcal{Q}$ and outputting $\sqrt{1-\delta}x + \sqrt{\delta}y$.
For our sparse linear regression result we need $\delta \leq ck$ for some small enough constant $c>0$ such that $U_\delta \mathcal{P}$ continues to encode a negative spike of magnitude $1 - \Theta(1/k)$.

\subsection{LD lower bound for negative-spike sparse Wishart model}

In this section, we provide evidence of hardness for the sparse negative-spike Wishart model, based on the low-degree likelihood ratio (LDLR) method. 

We show that the low-degree likelihood ratio is bounded for the sparse negative-spike Wishart model in the hard regime conjectured in~\cref{conj:hardSSW}.

Our proof follows closely the proof of~\cite{ComplexityRIP}. The main differences are that in our planted distribution the prior of $\bm\secret$ is the uniform distribution over $k$-sparse vectors (instead of a sparse Rademacher distribution, in which the exact sparsity can vary) and that the first coordinate of $\bm \secret$ is non-zero.
\begin{theorem}[Low-degree lower bound for $\nspca_{\theta}$]\label{thm:LDLRWishart}
    Fix $\delta \in (0, 0.1]$.
    Suppose that $n = o(\min(d,k^{2-\delta}))$.
    Consider the planted distribution $\mathcal{P}_n$ formed from $n$ samples from $\cP$ and null distribution $\mathcal{Q}_n$ formed from $n$ samples from $\cQ$ in $\nspca_{\theta}$ defined in \cref{def:spiked_wishart}.
    Then for any degree-$k^{\delta}$ polynomial $f(z_1,z_2,\ldots,z_n): \R^{n\times d}\to \R$ such that $\E_{\cQ} f=0$, we have
    \begin{equation*}
        \frac{\E_{\mathcal{P}_n} f}{\sqrt{\E_{\mathcal{Q}_n} f^2}}\leq 1+o(1)\,.
    \end{equation*}
\end{theorem}

\begin{proof}
    For degree-$D$ polynomials, the maximum value of $\frac{\E_{\mathcal{P}_n} f}{\sqrt{\E_{\mathcal{Q}_n} f^2}}$ is given by (see~\cite{HopkinsThesis}) \[\Norm{L_{n,\gamma,\theta,\chi}}^2\coloneqq \E\Paren{\Paren{\frac{d{\mathcal{P}_n}}{d{\mathcal{Q}_n}}}^{\leq D}}^2\,.\]

    Equations (8) and (10) in~\cite{ComplexityRIP} provide the bound
    \begin{align*}
    \Norm{L_{n,\gamma,\theta,\chi}}^2
    &\leq \sum_{\ell=0}^{\lfloor D/2 \rfloor} \frac{(2n+4D)^\ell}{\ell!}\cdot \frac{\theta^{2\ell}}{4^\ell}\cdot \E \iprod{\bm x^{(1)},\bm x^{(2)}}^{2\ell}\\
    &\leq 1 + \sum_{\ell=1}^{\lfloor D/2 \rfloor}  
    \Paren{\frac{4\theta^2 \max(n, D)}{\ell}}^\ell \E\iprod{\bm x^{(1)},\bm x^{(2)}}^{2\ell} \\
    &\leq 1 + \sum_{\ell=1}^{\lfloor D/2 \rfloor}  
    \Paren{\frac{4 \max(n, D)}{\ell}}^\ell \E\iprod{\bm x^{(1)},\bm x^{(2)}}^{2\ell}\,,
    \end{align*}
    where $\bm x^{(1)}, \bm x^{(2)}$ are independently sampled from the prior of the spike vector in the planted distribution, i.e.,
    the first coordinates $\bm x_1^{(1)}, \bm x_1^{(2)}$ are sampled uniformly from $\Set{\pm \sqrt{\frac{1}{k+1}}}$ and the rest of the coordinates $\bm x_{\setminus 1}^{(1)}, \bm x_{\setminus 1}^{(2)}$ are sampled uniformly from $k$-sparse $\Set{\pm \sqrt{\frac{1}{k+1}},0}^d$ vectors.

    Now we note that
     \begin{align*}
    \E\iprod{\bm x^{(1)}, \bm x^{(2)}}^{2\ell} &=  
    \E\Paren{\iprod{\bm x_{\setminus 1}^{(1)},\bm x_{\setminus 1}^{(2)}}+ \bm x_1^{(1)}\bm x_1^{(2)}}^{2\ell}  \\
    & \leq 2^{2\ell} \cdot \Paren{\E\iprod{\bm x_{\setminus 1}^{(1)},\bm x_{\setminus 1}^{(2)}}^{2\ell}+ \E \Paren{\bm x_1^{(1)}\bm x_1^{(2)}}^{2\ell}} \\
    & = 2^{2\ell}  \cdot \E\iprod{\bm x_{\setminus 1}^{(1)},\bm x_{\setminus 1}^{(2)}}^{2\ell}
    + \Paren{\frac{4}{(k+1)^2}}^\ell\,,
     \end{align*}
     so 
     \begin{align*}
        \Norm{L_{n,\gamma,\theta,\chi}}^2
        &\leq 1 + \sum_{\ell=1}^{\lfloor D/2 \rfloor} \Paren{\frac{16 \max(n, D)}{\ell}}^\ell \cdot \E\iprod{\bm x_{\setminus 1}^{(1)},\bm x_{\setminus 1}^{(2)}}^{2\ell} + \sum_{\ell=1}^{\lfloor D/2 \rfloor} \Paren{\frac{16\max(n,D)}{(k+1)^2\ell}}^\ell\,.
     \end{align*}

Using that $n = o(k^2)$ and $D = o(k^2)$, we have
 \begin{equation*}
    \sum_{\ell=1}^{\lfloor D/2 \rfloor} \Paren{\frac{16\theta^2 \max(n, D)}{(k+1)^2\ell}}^\ell \leq o(1)\,.
 \end{equation*}

For the remaining term, we have by~\Cref{lem:exp-bound} that there is an absolute constant $C > 0$ such that
\begin{align*}
    \sum_{\ell=1}^{\lfloor D/2 \rfloor} \Paren{\frac{16\max(n,D)}{\ell}}^\ell\cdot \E\iprod{\bm x_{\setminus 1}^{(1)},\bm x_{\setminus 1}^{(2)}}^{2\ell}
    &\leq \sum_{\ell=1}^{\lfloor D/2 \rfloor} \Paren{C\cdot \max(n,D) \cdot \Paren{\frac{1}{d} + \frac{\ell}{k^2}}}^\ell\,.
\end{align*}
For this sum to be bounded by $o(1)$, we want both $\max(n, D) = o(d)$ and $\max(n, D) = o(k^2/D)$. This is achieved if $n = o(d)$ and $nD = o(k^2)$ and $D = o(k)$, so it is achieved under the conditions of the theorem with $D \leq k^{\delta}$.
Then 
\[\Norm{L_{n,\gamma,\theta,\chi}}^2 \leq 1 + o(1)\,.\]

\end{proof}

\begin{lemma}[Expectation bound]\label{lem:exp-bound}
    In the setting of~\Cref{thm:LDLRWishart}, there exists an absolute constant $C > 0$ such that
    \begin{equation}
        \E\iprod{{\bm x}_{\setminus 1}^{(1)}, \bm x_{\setminus 1}^{(2)}}^{2\ell} \leq (C\ell)^{\ell} \Paren{\frac{1}{d} + \frac{\ell}{k^2}}^\ell\,. 
    \end{equation}
\end{lemma}
\begin{proof}
Let $\bm S_1, \bm S_2 \subset [d]$ be the sets of non-zero indices of $\bm x_{\setminus 1}^{(1)}$ and $\bm x_{\setminus 1}^{(2)}$ respectively. We have 
\begin{align*}
    \E\iprod{{\bm x}_{\setminus 1}^{(1)}, \bm x_{\setminus 1}^{(2)}}^{2\ell}
    &= \frac{1}{(k+1)^{2\ell}} \E_{\bm S_1, \bm S_2} \E\Brac{\Paren{\sum_{i \in \bm S_1 \cap \bm S_2} \bm \rho_i}^{2\ell} \Biggr| \bm S_1, \bm S_2}\,,
\end{align*}
where $\bm \rho_i$ are \iid Rademacher variables.
Then $\sum_{i \in \bm S_1 \cap \bm S_2} \bm \rho_i$ is a sub-Gaussian random variable with variance proxy $|\bm S_1 \cap \bm S_2|$, so for an absolute constant $C > 0$,
\begin{align*}
    \E\iprod{{\bm x}_{\setminus 1}^{(1)}, \bm x_{\setminus 1}^{(2)}}^{2\ell}
    &\leq \frac{(C\ell)^{\ell}}{(k+1)^{2\ell}} \E \Abs{\bm S_1 \cap \bm S_2}^{\ell}\,.
\end{align*}
Let $\bm a_i \in \{0, 1\}$ for $i \in [k]$ be the indicator that the $i$-th element of $\bm S_1$ is also in $\bm S_2$. Then $\Abs{\bm S_1 \cap \bm S_2} = \bm a_1 + \ldots + \bm a_k$. Also let $\bm b_i$ for $i \in [k]$ be \iid $\mathrm{Ber}(\tfrac{k}{d})$ variables. Then for any $S \subseteq [k]$ we have
\[\E \prod_{i \in S} \bm a_i = \frac{k}{d} \frac{k-1}{d-1} \ldots \frac{k-|S|+1}{d-|S|+1} \leq \Paren{\frac{k}{d}}^{|S|} = \E \prod_{i \in S} \bm b_i\,.\]
Then also $\E (\bm a_1 + \ldots + \bm a_k)^\ell \leq \E \bm B^\ell$ where $\bm B$ is a $\mathrm{Bin}(k, \tfrac{k}{d})$ variable.
\cite[Corollary 1]{ahle2022sharp} gives that $\E \bm B^{\ell} \leq \Paren{\tfrac{k^2}{d} + \tfrac{\ell}{2}}^{\ell}$.
Therefore, overall 
\[\E\iprod{{\bm x}_{\setminus 1}^{(1)}, \bm x_{\setminus 1}^{(2)}}^{\ell} \leq \frac{(C\ell)^{\ell}}{(k+1)^{2\ell}} \Paren{\frac{k^2}{d} + \frac{\ell}{2}}^{\ell} \leq (C\ell)^{\ell} \Paren{\frac{1}{d} + \frac{\ell}{k^2}}^\ell\,.\]
\end{proof}

\section{Statistical query lower bound for negative-spike sparse Wishart model}\label{sec:SQWishart}

In this section, we prove a statistical query (SQ) lower bound for negative spike sparse Wishart model.
Our proof heavily relies on an almost equivalence established between LDLR lower bounds and SQ lower bounds~\cite{brennan2020statistical}, and resembles the proof in section 8.3 of their paper. 

\subsection{SQ framework}

The SQ framework is a restricted computational model where a learning algorithm can make certain types of queries to an oracle and get answers that are subject to a certain degree of noise~\cite{kearns1998efficient}.
We will focus on the SQ model with VSTAT queries which is used in~\cite{brennan2020statistical}, where the learning algorithm has access to the VSTAT oracle as defined below:
\begin{definition}[VSTAT oracle]
\label{VSTAT}
    Given a query $\phi: \R^d \rightarrow [0, 1]$ and a distribution $D$ over $\R^d$, the VSTAT(n) oracle returns $\E_{x \sim D} [ \phi (x) ] + \zeta$ for an adversarially chosen $\zeta \in \R$ such that $|\zeta| \leq \max \Big(\frac{1}{n}, \sqrt{\frac{\E [\phi] (1-\E [\phi])}{n}} \Big)$.
\end{definition}
One way to show an SQ lower bound is by computing the statistical dimension of the hypothesis testing problem, which is a measure on the complexity of the testing problem. 
In this paper, we use the following definition of statistical dimension introduced by~\cite{feldman2017statistical}:
\begin{definition}[Statistical dimension]
Let $\mu_{\emptyset}$ be some distribution with $\mathcal{D}_{\emptyset}$ as density function. Let $\cS=\Set{\mu_u}$ be some family of distributions indexed by $u$, such that 
$\mu_u$ has density function given by $\mathcal{D}_u$. 
Consider the hypothesis testing problem between:
\begin{itemize}
    \item Null hypothesis: sample \iid from $\mathcal{\mu}_{\emptyset}$, 
    \item Alternative hypothesis: sample \iid from $\mathcal{\mu}_{u}$ where $u$ is sampled from some prior distribution $\mu$.
\end{itemize}
 For $D_u \in \cS$, define the relative density $\Bar{D}_u(x) = \frac{D_u(x)}{D_{\emptyset}(x)}$ and the inner product $\iprod{f, g} = \E_{x \sim D_{\emptyset}} [ f(x) g(x) ]$. The statistical dimension $SDA(\cS, \mu, n)$ measures the tail of $\iprod{\Bar{D}_u, \Bar{D}_v} - 1$ with $u$, $v$ drawn independently from $\mu$:
\begin{align*}
    &SDA(\cS, \mu, n)\\
    &= \max \Big\{ q \in \N : \E_{u, v \sim \mu} \Bigbrac{|\iprod{\Bar{D}_u, \Bar{D}_v} - 1 | \mid A} \leq \frac{1}{n} \text{ for all events A s.t.} \Pr_{u, v \in \mu} (A) \geq \frac{1}{q^2} \Big\}\,.
\end{align*}
\end{definition}
We will use $SDA(n)$ or $SDA(\cS, n)$ when $\cS$ or $\mu$ are clear from the context.
In~\cite{feldman2017statistical}, it was shown that the statistical dimension is a lower bound on the SQ complexity of the hypothesis test using VSTAT oracles:
\begin{theorem}[Theorem 2.7 of \cite{feldman2017statistical}, Theorem A.5 of \cite{brennan2020statistical}]
\label{theorem_sqDimLB}
    Let $D_{\emptyset}$ be a null distribution and $\cS$ be a set of alternative distributions. Then any (randomized) statistical query algorithm which solves the hypothesis testing problem between $D_{\emptyset}$ and $\cS$ with probability at least $1-\gamma$ requires at least $(1-\gamma)SDA(\cS, n)$ queries to $VSTAT(\frac{n}{3})$.
\end{theorem}

The almost equivalence between SQ lower bounds and low-degree lower bounds is established in~\cite{brennan2020statistical}:
\begin{theorem}[Theorem 3.1 in \cite{brennan2020statistical}, LDLR to SDA]\label{thm:almost-equivalence}
    Let $\ell \in \mathbb{N}$ with $\ell$ even and $\mathcal{S}=\left\{D_{v}\right\}_{v \in S}$ be a collection of probability distributions with prior $\mu$ over $\mathcal{S}$. Suppose that $\mathcal{S}$ satisfies:
    \begin{itemize}
        \item The $\ell$-sample high-degree part of the likelihood ratio is bounded by $\left\|\mathbf{E}_{u \sim \mathcal{S}}\left(\bar{D}_{u}^{>\ell}\right)^{\otimes \ell}\right\| \leqslant \gamma$.
        \item For some $n \in \mathbb{N}$, the degree-$\ell$ likelihood ratio is bounded by $\left\|\mathbf{E}_{u \sim \mathcal{S}}\left(\bar{D}_{u}^{\otimes n}\right)^{\leqslant \ell}-1\right\| \leqslant \varepsilon$.
    \end{itemize}
   Then for any $q \geqslant 1$, it follows that
   $$
   \operatorname{SDA}\left(\mathcal{S}, \frac{n}{q^{2 / \ell}\left(\ell \varepsilon^{2 / \ell}+\gamma^{2 / \ell} n\right)}\right) \geqslant q\,.
   $$
    
   \end{theorem}

\subsection{SQ lower bound for sparse negative-spike Wishart model}
The sparse negative-spike Wishart model $\nspca_{\theta}$ under the statistical query framework has the following form:
\begin{model}[Sparse spiked Wishart model in SQ model]
    \label{def:sq_spiked_wishart}
    Let $d, k \in \N$ with $d \geq k$ and $\theta \in (0, 1)$.
    We define the $\nspca_{\theta}$ problem in the statistical query case as the following distinguishing problem:
    We want to distinguish between the distributions $\cP$ or $\cQ$ defined as follows, using VSTAT oracle calls to them.
    \begin{itemize}
        \item \textbf{Planted  distribution $\mathcal{P}$}: Sample a unit vector $\bm x \in \R^{d+1}$ with $\bm x_1=-\tfrac 1 {\sqrt{k+1}}$ and $\bm x_{\setminus 1}$ sampled uniformly from $k$-sparse $\Set{\pm \tfrac 1 {\sqrt{k+1}},0}^d$ vectors and fix it for the rest of the sampling procedure.
        Produce samples by sampling from $N(0, \Id_{d+1}-\theta \cdot \bm x\bm x^\top)$.
        \item \textbf{Null distribution $\mathcal{Q}$}: $N(0, \Id_{d+1})$.
    \end{itemize} 
\end{model}

Corresponding to \cref{thm:almost-equivalence}, the set of distributions $\mathcal{S}$ is parameterized by the random unit vector $\bm x$, and is given by $\Set{N(0, \Id_{d+1}-\theta \cdot \bm x\bm x^\top)}$ in our setting.
The prior distribution $\mu$ is determined by the distribution of $\bm x$.

Under the statistical query framework, we have the following lower bound:

\begin{theorem}[SQ lower bound for $\nspca_\theta$]\label{thm:SQbound}
    Consider the $\nspca_\theta$ model. 
    Let $0 < \delta \leq 0.1$ be an arbitrary absolute constant.
    Then, for $n = o(\min(d/k^{\delta}, k^{2-2\delta}))$, we have the SQ lower bound $\text{SDA}(\mathcal{S}, n)\geq 2^{k^{\delta}}$.
\end{theorem}

To apply the almost equivalence relation between low-degree lower bounds and statistical query lower bounds, we need to prove the following lemma:
\begin{lemma}\label{lem:high-degree-bound}
    Consider the $\nspca_\theta$ model. 
    Fix $\delta \in (0, 0.1]$. 
    Then for any $\ell \leq k^{\delta}$, we have 
    \begin{equation*}
        \Norm{\E_{u \sim \mathcal{S}} \Paren{\bar{D}_u^{>\ell}}^{\otimes \ell}}^2
    \leq k^{-\ell^2/3}\,.
    \end{equation*}  
\end{lemma}
\begin{proof}
    We follow the same proof as Lemma 8.21 in~\cite{brennan2020statistical}\footnote{Particularly, we set their $k$ and $d$ to $\ell$, and $n$ to $d$ in our setting.}, and get
    \begin{equation*}
        \Norm{\E_{u \sim \mathcal{S}} (\bar{D}_x^{>\ell})^{\otimes \ell}}^2
        \leq \ell^{2\ell^2}\Paren{1-4\theta^2}^{2\ell^2}\E \iprod{\bm x^{(1)},\bm x^{(2)}}^{2\ell\cdot (\ell+1)}\,,
    \end{equation*}
    where $\bm x^{(1)},\bm x^{(2)}$ are independently sampled from the prior of the spike vector in the planted distribution, i.e.,
    the first coordinates $\bm x_1^{(1)},\bm x_1^{(2)}$ are sampled uniformly from $\Set{\pm \sqrt{\frac{1}{k+1}}}$ and the rest of the coordinates $\bm x_{\setminus 1}^{(1)},\bm x_{\setminus 1}^{(2)}$ are sampled uniformly from $k$-sparse $\Set{\pm \sqrt{\frac{1-\frac{1}{k+1}}{k}},0}^d$ vectors.

    Following the proof of~\cref{thm:LDLRWishart}, we have
    \begin{equation*}
        \E \iprod{\bm x^{(1)},\bm x^{(2)}}^{2\ell(\ell+1)}\leq 
        2^{2\ell(\ell+1)} \cdot (C\ell(\ell+1))^{\ell(\ell+1)} \cdot \Paren{\frac{1}{d}+\frac{\ell(\ell+1)}{k^2}}^{\ell(\ell+1)} + \Paren{\frac{4}{(k+1)^2}}^{\ell(\ell+1)}\,.
    \end{equation*}
    As a result, when $\ell\leq k^{\delta}$, we have
    \begin{equation*}
        \ell^{2\ell^2}\Paren{1-4\theta^2}^{2\ell^2}\E \iprod{\bm x^{(1)},\bm x^{(2)}}^{2\ell (\ell+1)}\leq (1/k^{1/3})^{\ell^2}\,,
    \end{equation*}
    which finishes the proof.
\end{proof}

\begin{proof}[Proof of \cref{thm:SQbound}]
    We apply~\cref{thm:almost-equivalence} with $n = o(\min(d, k^{2-\delta}))$, $\ell=k^{\delta}$ and $q=2^{\ell}$. 
    By~\cref{lem:high-degree-bound}, we can take $\gamma=k^{-\frac{\ell^2}{3}}$. 
    By~\cref{thm:LDLRWishart}, we can take $\epsilon=o(1)$. As a result, we have 
    $\text{SDA}(\mathcal{S},\frac{n}{100 k^{\delta}})\geq 2^{k^{\delta}}$.
\end{proof}

\end{document}